%% file: neurips_2024.tex
\documentclass{article}

\PassOptionsToPackage{numbers, compress}{natbib}

\usepackage[preprint]{neurips_2024}

\usepackage[utf8]{inputenc} 
\usepackage[T1]{fontenc}    
\usepackage{hyperref}       
\usepackage{url}            
\usepackage{booktabs}       
\usepackage{amsfonts}       
\usepackage{nicefrac}       
\usepackage{microtype}      
\usepackage{xcolor}         
\usepackage{booktabs}
\usepackage{multirow}
\usepackage{amsmath}
\usepackage{algorithm}
\usepackage{algpseudocode}
\usepackage{graphicx}
\usepackage{subcaption}
\usepackage{stfloats}
\usepackage{colortbl}
\usepackage{tcolorbox}
\usepackage{wrapfig}
\usepackage{soul}
\usepackage{amsmath}
\usepackage{amsthm}
\newtheorem{proposition}{Proposition}

\usepackage[accsupp]{axessibility}  
\usepackage{enumitem}

\setlist[itemize]{leftmargin=0.8 cm}

\usepackage{xspace}
\newcommand{\sys}{\textsc{GALOT}\xspace}

\title{GALOT: Generative Active Learning via Optimizable Zero-shot Text-to-image Generation}

\author{%
  Hanbin Hong \\
  University of Connecticut\\
  Storrs, CT 06269 \\
  \And
  Shenao Yan \\
  University of Connecticut\\
  Storrs, CT 06269 \\
  \And
  Shuya Feng \\
  University of Connecticut\\
  Storrs, CT 06269 \\
  \And
  Yan Yan \\
  Illinois Institute of Technology\\
  Chicago, IL 60616 \\
  \And
  Yuan Hong \\
  University of Connecticut\\
  Storrs, CT 06269 \\
}

\begin{document}

\maketitle

\begin{abstract}
Active Learning (AL) represents a crucial methodology within machine learning, emphasizing the identification and utilization of the most informative samples for efficient model training. However, a significant challenge of AL is its dependence on the limited labeled data samples and data distribution, resulting in limited performance. To address this limitation, this paper integrates the zero-shot text-to-image (T2I) synthesis and active learning by designing a novel framework that can efficiently train a machine learning (ML) model sorely using the text description. Specifically, we leverage the AL criteria to optimize the text inputs for generating more informative and diverse data samples, annotated by the pseudo-label crafted from text, then served as a synthetic dataset for active learning. This approach reduces the cost of data collection and annotation while increasing the efficiency of model training by providing informative training samples, enabling a novel end-to-end ML task from text description to vision models. Through comprehensive evaluations, our framework demonstrates consistent and significant improvements over traditional AL methods. Codes will be released upon acceptance.
\end{abstract}

\section{Introduction}
\label{sec:Introduction}

\input{sec/Introduction}

\section{Related Works}
\label{sec:RelatedWorks}
\input{sec/RelatedWorks}

\section{Generative AL via Conditional Diffusion Models}
\label{sec:Methodology}
\input{sec/Methodology}

\section{Experimental Results}
\label{sec:Experiment}
\input{sec/Experiment}

\section{Conclusion}
\input{sec/Conclusion}

\bibliographystyle{splncs04}
\bibliography{main}

\clearpage
\appendix
\input{sec/Appendix}

\end{document}

%% file: sec/Introduction.tex
\begin{wrapfigure}{r}{0.50\textwidth}
    \centering
    \vspace{-8pt}
    \includegraphics[width=\linewidth]{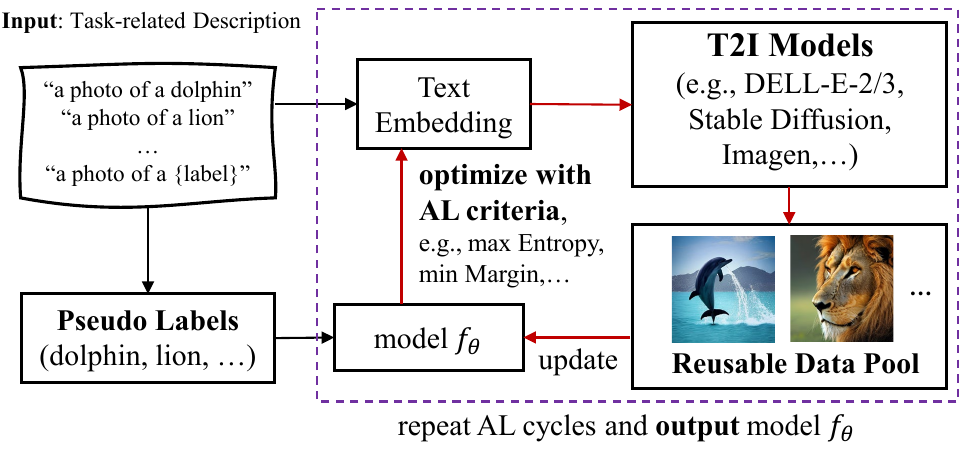}
    \vspace{-14pt}
    \caption{Overview of GALOT. The task-related text is first converted into the text embedding. Then, it iteratively executes 1) optimizing the text embedding according to the AL criteria based on the current model's output, 2) generating data samples with optimized text embedding, and 3) training the model with generated data and pseudo label. \sys train the vision model from text inputs.}
    \label{fig:overview}
    \vspace{-18pt}
\end{wrapfigure}


Active learning is a pivotal technique in machine learning that focuses on selecting and using the most informative samples from a large dataset to train models efficiently \cite{ALSurvey,ALsurvey2022}. This method significantly reduces the need for labeled data, which lowers labeling costs and speeds up training. However, despite its successes, current active learning strategies are often confined to predetermined, limited data distributions. This confinement limits the exploration of varied and potentially more insightful data points \cite{GAAL,BGAL,LADA}.


Zero-shot text-to-image synthesis models like Stable Diffusion, DELL-E-2, and Imagen provide a novel approach to generating out-of-distribution data. By transforming textual descriptions into visual representations, these models can create expansive datasets of unlabeled images \cite{StableDiffusion, DELL-E-2, Imagen}. They utilize advanced language models to explore a nearly unlimited embedding space, which is crucial for enhancing Active Learning (AL) strategies \cite{DiffusionSurvey2022, DiffusionSurvey2023}.


Nevertheless, despite its potential, not all text embeddings yield informative samples, highlighting the need for a selective strategy in choosing and optimizing text inputs. 
Most of text embeddings might be irrelevant to a specific task, such as utilizing medical images in a cat-dog classification. Moreover, relevant synthetic images may not invariably contribute meaningful information to the learning process \cite{GAAL}. To overcome these challenges, we narrow the scope of text down to task-relevant classes and refine text embedding to produce inherently informative samples (see \autoref{fig:overview}).

The subsequent challenge emerges in the form of annotating the synthetically generated images, a process that traditionally demands substantial manual efforts and resources. Advancements in pre-trained text-to-image models provide a valuable solution. By generating images that accurately reflect their source text, we can assign ``pseudo labels'' to these images. This approach eliminates the need for manual labeling, thereby reducing labeling costs. Our experiments confirm the high accuracy of these pseudo labels, as detailed in Section \ref{sec:Performance of Pseudo-labeling via Human Annotation}.

Embracing zero-shot text-to-image generation and addressing the identified challenges, our work introduces a novel active learning framework: \underline{G}enerative \underline{A}ctive \underline{L}earning via \underline{O}ptimizable Zero-shot \underline{T}ext-to-image Generation (\sys). \sys offers significant advantages over traditional active learning methods:



\noindent\textbf{Data and Annotation Efficiency}. By utilizing synthetic data and pseudo-labels, \sys significantly reduces the need for annotated data. Remarkably, \sys enables training vision models directly from text inputs, pioneering a \emph{text-to-model} approach (see \autoref{fig:overview} for illustration). Utilizing accessible online text-to-image models, such as DELL-E-3 \cite{DELL-E-3}, lowers barriers for users, allowing them to efficiently train vision models for diverse applications.
  
  
\noindent \textbf{Universally-enhanced Learning through Diverse Data Sources}. The unlimited variety of text inputs offers a rich source of potentially informative examples. When combined with the active learning strategies, this diversity consistently boosts learning efficiency, as validated in our experiments (Section \ref{sec:Experiment}).


\noindent \textbf{Data Reusability and Model Transferability}. Despite the initial computational investment, the synthetic dataset can be repurposed across various vision models, often enhancing learning performance without data re-generation (refer to Section \ref{sec: Dataset Reuse} for experimental insights).
    


Therefore, in this paper, we made the following primary contributions: 

\vspace{-0.08in}

\begin{itemize}
     
\item To our best knowledge, we propose the first paradigm of ``\emph{Generative Text-to-Image Active Learning}'' that harmonizes active learning with zero-shot text-to-image synthesis. This approach generates informative dataset using textual inputs for training vision models, pioneering a text-to-model approach.

\vspace{-0.05in}

\item We propose a novel algorithm that practically refines text inputs and optimizes text embeddings. It leverages the gradients of diffusion models to update the text embedding for generating informative, task-specific data samples.

\vspace{-0.05in}

\item We conduct comprehensive experiments on three widely-used datasets under various settings to demonstrate the effectiveness of \sys. The experimental results, benchmarked with state-of-the-art (SOTA) AL methods, substantiate the efficacy of our proposed generative active learning framework in different aspects.



\end{itemize}


%% file: sec/RelatedWorks.tex


Active learning (AL) is a machine learning strategy aimed at improving model accuracy with fewer labeled instances by allowing the model to select the most informative data points for annotation. It is categorized into three types \cite{ALSurvey}: 1) Pool-based sampling, where the model selects key instances from a large pool of unlabeled data for labeling \cite{ALSurvey,UncertaintyDropout,CoreSet,LLAL}. 2) Stream-based selective sampling, where instances are evaluated sequentially, and the model determines whether to request labels from an oracle \cite{SSSfirstwork,ALDS,LAAL}. 3) Membership query synthesis, where the model generates synthetic instances to be labeled by the oracle \cite{MQSfirstwork,GAAL,BGAL,ASAL}.
Our approach merges membership query synthesis with pool-based sampling, generating data guided by existing instances.



\noindent \textbf{Pseudo-labeling}. Pseudo-labeling uses model predictions to label unlabeled data, enhancing training efficiency. However, predictions can be inaccurate, especially in undertrained models, potentially introducing errors into the training set \cite{CEAL, PseudoLabel}. Our approach differs by utilizing off-the-shelf text-to-image (T2I) models to generate accurate pseudo-labels directly from text inputs.


\noindent \textbf{Data Augmentation and Generative Active Learning}. Data augmentation involves creating variations of labeled samples, such as rotations and flips, to improve model generalization and performance \cite{FlipDA}. Techniques also include integrating adversarial examples into the training set, enlarging it without additional labeling costs \cite{PseudoLabel}. However, these augmented data instances may not necessarily be informative. Generative Active Learning (GAL) enhances traditional active learning by using generative models like GANs to create synthetic queries for annotation. GAAL \cite{GAAL} pioneered this by using GANs to synthesize samples for active learning queries. BGADL \cite{BGAL} trains a GAN on selected labeled data, then integrates these generated samples into the active learning process without additional annotation costs. In contrast, ASAL \cite{ASAL} uses GANs to generate informative, though not necessarily realistic, samples and then matches them with similar instances from an unlabeled pool. Despite these innovations, studies show that GAL does not consistently match the effectiveness of traditional active learning methods \cite{GAAL,ASAL} and is often limited to specific machine learning models like Bayesian neural networks \cite{BGAL} or Support Vector Machines \cite{GAAL}. A fundamental issue is that GANs, trained on the distribution of the unlabeled pool, often produce lower-quality data that does not introduce novel, information-rich instances.

\noindent \textbf{Text-to-image Generation}. T2I synthesis translates textual descriptions into visual representations, with Reed et al. \cite{FirstGANT2I} marking a significant initial advancement by integrating textual semantics into the visual generative process. Subsequent improvements, such as stacked \cite{StackGAN}, attentional \cite{AttnGAN}, and controllable GANs \cite{ControlGAN}, enhanced image resolution and quality, though challenges like mode collapse persist \cite{GAN}. Alternative approaches using autoregressive models like DELL-E, Cogview, and Nvwa \cite{DELL-E,Cogview,Nvwa} also demonstrate effective image synthesis.

A major breakthrough in T2I has been achieved with diffusion models, which convert noise into detailed images through a denoising process \cite{DDPM,DPM}. Models like Stable Diffusion, DALL-E-2/3, and Imagen \cite{StableDiffusion,DELL-E-2,DELL-E-3,Imagen} excel in generating high-definition, realistic images and show robust zero-shot capabilities, enabling the creation of accurate images from new textual prompts. Despite their high computational demands, these models have revolutionized data synthesis and augmentation \cite{TTIDA,effectDA}, providing a foundation for generating diverse image datasets from text.

%% file: sec/Methodology.tex



\begin{wrapfigure}{r}{0.50\textwidth}
    \centering
    \vspace{-8pt}
    \includegraphics[width=\linewidth]{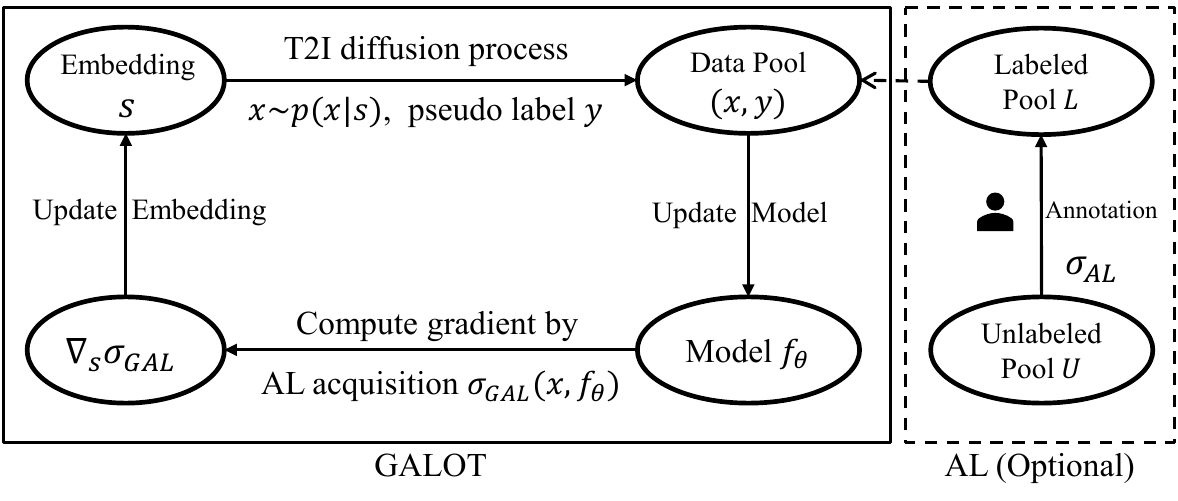}
    \caption{GALOT Workflow for Each Active Learning Cycle. In each active learning cycle, the data sample is generated using pre-trained T2I models with the embedding $s$ and pseudo label. Optionally combined with the traditional AL with real datasets as a complement, the generated data can be used to train a model. The embedding is then optimized according to the updated model via the gradients of the AL acquisition.}
    \label{fig:work flow}
    \vspace{-.2in}
\end{wrapfigure}

\subsection{Active Learning with Data Synthesis}

Consider a machine learning model, \( f_\theta: \mathbb{R}^d \rightarrow \mathcal{Y}\), where \( \theta \) denotes the model parameters. Given a loss function \( \mathcal{L} \), the model can be trained using samples \( x\in \mathbb{R}^d \) drawn from a data distribution conditioned on \( s \), i.e., \( p(x|s) \). The condition \(s\) and generated sample \(x\) can be in any format, e.g., text, audio, and image, which corresponds to different types of off-the-shelf guided diffusion models \cite{DBLP:conf/icml/Glide,DBLP:conf/icml/TexttospeechICML2021,DBLP:conf/iclr/Textto3DICLR2023}, which enable the zero-shot generative active learning in different domains. In this paper, we take the text-to-image generation as an example to pursue generative active learning in the image domain. The goal of active learning is to train the model \(f_\theta\) from scratch using as few labeled data points as possible. Different from traditional active learning \cite{CEAL,ALSurvey,LAAL}, our model is trained on the data samples drawn from the data distribution \(p(x|s)\). 

The optimization problem for vision model training can be formulated as:
\vspace{-5pt}
\begin{equation}
\theta = \arg\min_{\theta} \mathbb{E}_{x \sim p(x|s)} \mathcal{L}(f_\theta(x), y)
\end{equation}
\vspace{-2pt}
where \( y \in \mathcal{Y}\) denotes the ground-truth label of the sample \(x\). The goal of the training process is to minimize the loss over the generated sample-label pairs. 

\subsection{Optimizable Text-to-image Synthesis}

Due to the uncertainty when sampling the random data sample from the data distribution and the infinity possibility of the condition \(s\), the unlabeled data pool $\mathcal{U}_{x\sim p(x|s)}$ constructed by the data samples can be considered as an infinite set if the computational resource for sampling data can be infinite. However, not any conditions can generate informative samples that are beneficial for the training, e.g., an image generated from ``flower'' cannot help distinguish the ``dog'' and ``cat''. Therefore, the condition should be limited to the same domain as the vision task. Also, due to the high computational demand of data synthesis, we aim to generate as few samples as possible by increasing the information in the data samples, which is similar to the goal of active learning.

Traditional active learning methods \cite{ALSurvey} utilize an acquisition function, \( \sigma(\cdot) \), such as the Shannon Entropy \cite{shannon}, to prioritize the selection of informative training samples. Inspired by this idea, we propose to leverage the acquisition function to optimize the text condition to ensure the generated samples are informative. Specifically, we optimize the text embedding to maximize the acquisition function while constraining the text embedding related to the vision task. The text optimization problem can be formally written as:
\begin{equation}
s = \arg\max_{s} \mathbb{E}_{x \sim p(x|s)} \sigma(x,f_\theta) \quad \text{s.t.} \quad \|s - s^*\|_2 < \epsilon
\end{equation}
where \( s^* \) is a predefined condition related to the specific vision task at hand, and \( \epsilon \) serves as a regularization term to ensure that the optimized condition does not deviate significantly from the task-related condition.

\vspace{0.05in}

\noindent\textbf{Predefine Condition}. The construction of the predefined text condition \(s^*\) is intuitive. Since we are doing the text-to-image data synthesis, we can simply take the label name as the text to compute the text embedding as the predefined text condition. For example, when training a model to classify the CIFAR10 \cite{CIFAR10}, we transform the label names ``airplane'', ``automobile'', ``ship'', etc., to the text embedding for generating corresponding images in these classes. However, in practice, to generate high-quality images using the off-the-shelf T2I models, the text input is supposed to be as detailed as possible. Therefore, we can design some templates for different vision tasks to generate high-quality images, e.g., in our experiments, ``a realistic photo of a ship'' performs better than simply ``ship'' (see Section \ref{sec: abl exp}). Given the template $\tau$, the text condition \(s_i^*\) for class \(i\) can be computed by a transformation \(s_i^*=h_\tau(y_i)\), where \(y_i\) is the label of class \(i\). For example, given the template ``a photo of a \{label\}'' and the label name ``dog'', we simply replace the ``\{label\}'' with ``dog'', and then using CLIP \cite{CLIP} to encode the text input as text embedding.

\vspace{0.05in}

\noindent\textbf{Pseudo Label}. One of the challenges in generative active learning is the annotation of the generated data. We show that using the text-to-image model for the data synthesis can address this challenge, resulting in annotation-free active learning. Thanks to the advanced zero-shot capability of existing off-the-shelf T2I models \cite{DELL-E-2,DELL-E-3}, the generated image can accurately represent the text input. Therefore, we use the label \(y_i\) as the pseudo label for the images generated by \(s\). In the experiments, we show that the accuracy of the pseudo label can be as high as \(100\%\) under appropriate templates and embedding distortion (see Section \ref{sec:Performance of Pseudo-labeling via Human Annotation} for the experimental results).

\subsection{Text Optimization on Diffusion Models}

Diffusion models \cite{DDPM,DPM} typically include two processes for training and generating images, i.e., the forward diffusion process and the reverse diffusion process with each containing multiple steps. The forward diffusion process is used to prepare the target for training the diffusion model. With the reparameterization trick in \cite{DDPM}, the variable at the time step \(x_t\) can be expressed as:
\vspace{-5pt}
\begin{equation}
\label{eq:1}
    x_t=\sqrt{\overline{\alpha_t}}x_0 + \sqrt{1-\overline{\alpha_t}}\epsilon
\vspace{-5pt}
\end{equation} where \(\overline{\alpha_t}\) is the schedule parameter at \(t\), and \(\epsilon\sim \mathcal{N}(0,I)\).

The reverse diffusion process is adopted in the inference stage for generating images, where the variable at time step \(t-1\) is computed from the variable at time step \(t\) by denoising. Under the condition \(s\), the reverse diffusion process at each time step \(t\) can be defined as $x_{t-1}\sim p(x_{t-1}|x_t,s)$, and the final reversed variable (generated image) \(x_0\) can be expressed as:
\vspace{-5pt}
\begin{equation}
    x_0 \sim p(x_{0:T-1}|x_T,s)=\prod_{t=T}^1 p(x_{t-1}|x_t,s) 
\vspace{-5pt}
\end{equation} 
where \(T\) denotes the total time step.

Suppose we have white-box access to the pre-trained T2I diffusion models, to optimize the text embedding, one promising solution is to use the gradient descent, e.g., Projected Gradient Descent (PGD) \cite{PGD}. In this way, we update the text embedding \(s\) according to the gradients:
\vspace{-5pt}
\begin{align}
    s_i = &s_{i-1}+\alpha \ sgn[\mathbb{E}_{x_0\sim p(x_{0:T-1}|x_T,s)}\nabla_s\sigma(x_0,f_\theta)] \label{eq:text opt} \\ 
    & s.t. \ \ \|s_i-s^*\|_2 \leq \epsilon 
\vspace{-5pt}
\end{align} where \(i=1,2,...,k\) denotes the updating step, \(\alpha\) denotes the step size, \(sgn[\cdot]\) denotes the sign function. However, computing the gradient  \(\nabla_s\sigma(x_0,f_\theta)\) is extremely challenging even with the auto-gradient techniques in ML libraries, e.g., Pytorch \cite{pytorch} or TensorFlow \cite{tensorflow}, due to the recurrent denoising steps. By Proposition \ref{prop:1}, we make it feasible to estimate the gradients using the auto-gradient.

\begin{proposition}
\label{prop:1}
    Assume Eq. (\ref{eq:1}) holds in the reverse diffusion process, then the gradient can be written as:
\vspace{-5pt}
    \begin{equation}
        \nabla_s\sigma(x_0,f_\theta)= T \ \nabla_{x_0} \sigma \cdot J_{x_0,s}
        \vspace{-5pt}
    \end{equation} where \(J_{x_0,s}\) denotes the Jacobian of \(x_0\) w.r.t. \(s\).
\end{proposition}

\begin{proof}
    
Note that each reverse diffusion step is conditioned on the text condition \(s\), so the gradient should be computed as the summation:
\vspace{-10pt}
\begin{align}
\footnotesize
    \nabla_s\sigma(x_0,f_\theta) & = \sum_{t=0}^T \nabla_{x_t} \sigma \cdot J_{x_t,s} \\
    & = \sum_{t=0}^T \nabla_{x_0} \sigma \cdot J_{x_0, x_t} \cdot J_{x_t,x_0} \cdot J_{x_0,s} \\ 
    (\text{by Eq. (\ref{eq:1})}) & = T \ \nabla_{x_0} \sigma \cdot J_{x_0,s} \label{eq:10}
\end{align}

Thus, this completes the proof. 
\end{proof}

\begin{algorithm}[t]
\small
\caption{GALOT Training}
\label{alg:Algorithm}
\scriptsize
\begin{algorithmic}[1] 
    \State \textbf{Input:} cycle number $N$, sample number per cycle $B_{AL}$ for AL, sample number per cycle $B_{GAL}$ for GAL, labels $y$, text perturbation size $\epsilon$, GAL acquisition function $\sigma_{GAL}$, AL acquisition function $\sigma_{AL}$, text template $\tau$, label-to-text transformation $h$, classifier $f_\theta$, unlabeled training dataset $U$, T2I model $M()$, text update steps $n$, and text update stepsize $\alpha$.
    \State \textbf{Output:} model $f_\theta$
    \For{$\text{cycle} = 1$ \textbf{to} $N$}
        \State Initialize the labeled pool as $L=\emptyset$
        \State Select top $B_{AL}$ samples from $U$ ranked by $\sigma_{AL}$ as the batch $V$ and add to the labeled pool $L=L\cup V$ 
        \State Compute the predefined text embedding: $s^* = h_\tau(y)$
        
        \State Initialize text embedding $s = s^*$
        \State Optimize text embedding: $s' = \text{TextOpt}(s, s^*, \epsilon,\alpha,n, \sigma_{GAL}, f_\theta)$
    
        \State Generate $B_{GAL}$ input pairs $G = \{x, \hat{y}\}$ with $x = M(s')$ and pseudo label $\hat{y}=y$.
        \State Train the model $f_\theta$ with $L\cup G$
    \EndFor

\end{algorithmic}
\end{algorithm}

The gradients \(\nabla_{x_0} \sigma \cdot J_{x_0,s}\) can be efficiently computed by the auto-gradient through the graph of \emph{the last reverse diffusion step}. The assumption that Eq. (\ref{eq:1}) holds is based on the empirical observation that off-the-shelf T2I models are trained to approximate this equality in the reverse process.

\vspace{-12pt}
\subsection{GAL-AL Joint Sampling}
\vspace{-8pt}

Without consuming any annotation budget, our method can achieve moderate performance on zero-shot classification (see Section \ref{sec: joint sampling} for the experimental results), however, to train a model targeting the high performance on a test set in a specific task, the real training data can be helpful. Therefore, we combine the generated samples (without annotation) and the real samples (with annotation) to train the model, leading to significantly boosted performance on the test set compared with the model solely trained on the training set.

The algorithms for generative active learning are presented in Algorithm \ref{alg:Algorithm} (also see \autoref{fig:work flow} for the summarized workflow). \emph{TextOpt} function follows Eq. (\ref{eq:text opt}) to optimize the text embedding. We estimate the expectation of the gradients using the average gradients over $k$ generated samples.
\vspace{-.1in}



%% file: sec/Experiment.tex
In this section, we comprehensively evaluate the proposed generative active learning framework. First, we benchmark \sys with the SOTA AL baselines. Second, we evaluate the performance of \sys when the model is trained solely on the generated data without real data. Third, we evaluate the reusability and the transferability of the generated data. Fourth, we evaluate the performance of the pseudo-labeling by human annotation. Finally, we conduct 7 various ablation studies to evaluate the effectiveness of each component of \sys. Visualized examples can be found in Appendix \ref{sec:visual examples}.


\noindent \textbf{Datasets}. We evaluate our model using three established image datasets of varying scales: CIFAR10/100 \cite{CIFAR10}, and TinyImageNet \cite{tinyimagenet}. CIFAR10 includes 60,000 32x32 color images across 10 classes, suitable for basic image recognition tasks. CIFAR100 expands this to 100 classes with the same image count, offering a more complex classification challenge. TinyImageNet, a condensed version of the larger ImageNet dataset \cite{imagenet}, comprises 100,000 training, 10,000 validation, and 10,000 testing images in 64x64 resolution across 200 classes, providing a diverse and challenging environment for testing.


\noindent\textbf{Experimental Setting}. \sys is based on the off-the-shelf Text-to-Image models for generating data points. Specifically, we use Stable Diffusion \cite{StableDiffusion} pre-trained model (version 2.1 base \footnote{\url{https://huggingface.co/stabilityai/stable-diffusion-2-1-base}}) in all the experiments. Other pre-trained models like Imagen \cite{Imagen}, or other online platforms like DELL-E-3 \footnote{\url{https://openai.com/dall-e-3}} can also be used for zero-shot GAL w/wo the text optimization. In all the experiments, we set the diffusion step $T$ as $50$, and the sampling number for gradient estimation $k$ as $6$. The resolution of the generated images is $512\times512$ by default and further resized to match different datasets.  In all the experiments, unless otherwise stated, we use GAL-AL joint sampling, set ResNet18 as the classifier architecture, and set $\epsilon$ to linearly grow from $0$ to $0.5$. The detailed parameter setting for each section is summarized in Appendix \ref{apd: exp setting} for thorough reference and clarity. 

\noindent\textbf{Experimental Environment}. We implement \sys in PyTorch \cite{pytorch} based on open-source libraries Diffusers \footnote{ \url{https://github.com/huggingface/diffusers}} and DeepALPlus \footnote{\url{https://github.com/SineZHAN/deepALplus}} \cite{ALsurvey2022}. The experiments were running on a server with AMD EPYC Genoa 9354 CPUs (32 Core, 3.3GHz), and NVIDIA H100 Hopper GPUs (80GB each).
    
\input{tables/whole_compare}

\subsection{GALOT vs. SOTA Active Learning Methods}
\label{sec: main exp}

We first present a comprehensive comparison of model performance between \sys and state-of-the-art (SOTA) AL approaches (12 methods). To evaluate the efficiency of the text embedding refinement, we also evaluate a basic version of \sys, denoted as ``\sys (basic)'', where it only considers the basic text template without text optimization. For all three datasets, we perform $N=10$ cycle of active learning with a ResNet18 classifier (except for the GAAL \cite{GAAL} which only support SVM). The sampling number per cycle for the SOTA active learning methods is set to $B_{AL}=1,000$, resulting in a $10,000$ total annotation number. For \sys and its basic version, we additionally generate $B_{GAL}=|L|$ number of samples for each cycle without consuming the annotation budget. We measure the prediction accuracy on the test set of each dataset.


As is shown in \autoref{table:whole}, none of the SOTA methods shows dominant performance over other methods. However, our \sys and its basic version consistently outperforms SOTA methods in all the settings. This is a significant achievement, indicating the robustness and efficiency of \sys. The superior performance of \sys is due to the additional informative data generated by T2I models with pseudo-labels, leading to a minimum improvement of $0.72\%$ and a maximum improvement of $8.78\%$ over SOTA methods. The average improvement on CIFAR10, CIFAR100, and TinyImageNet are $3.84\%$, $7.71\%$, and $3.29\%$, respectively. This indicates our \sys can serve as an orthogonal method to boost the performance of active learning methods universally. It is worth noting that the \sys outperforms \sys (basic) by as much as $5.86\%$ on CIFAR10, indicating the superiority of the text refinement. On TinyImageNet, \sys and \sys (basic) give similar accuracies due to the high complexity of text embedding space (TinyImageNet has $200$ classes) but still demonstrate significant improvement over AL baselines.

\subsection{GAL-AL Joint Sampling}
\label{sec: joint sampling}

In this section, we benchmark our GAL-AL joint sampling (``AL+GAL'') with 1) Baseline AL: the traditional pool-based AL baseline, which solely relies on labeled samples for training; 2) Baseline GAL: the GAL baseline that solely relies on the generated samples for training; 3) Baseline Fully-supervised: the fully-supervised baseline that trains the model with full training dataset. Detailed settings on experimental parameters can be found in Appendix \ref{apd: exp setting}.

\begin{wrapfigure}{r}{0.4\textwidth}
    \centering
    \vspace{-8pt}
    \includegraphics[width=1.0\linewidth]{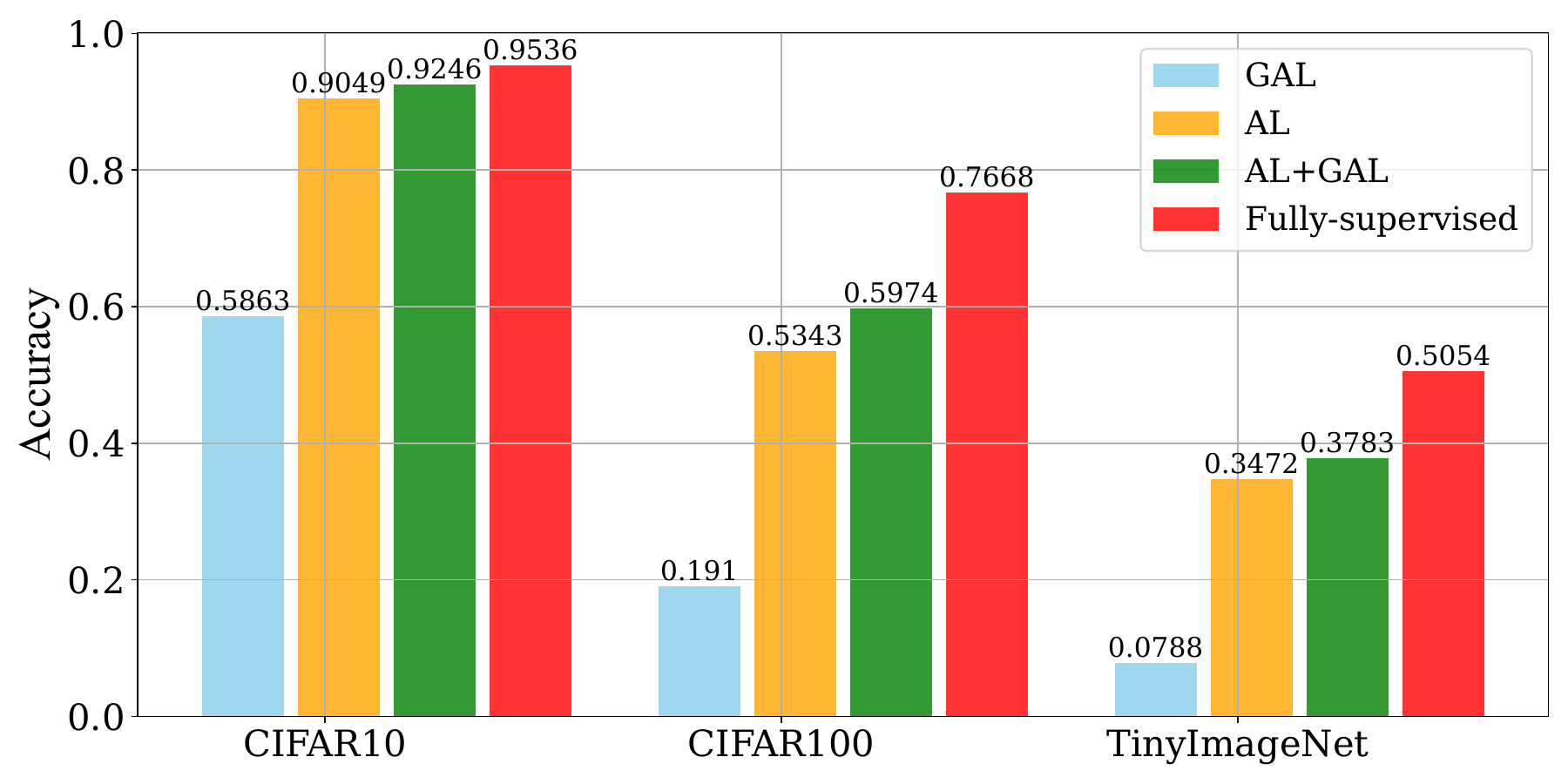}
    \vspace{-14pt}
    \caption{Comparison of Different Baselines.}
    \label{fig: baselines}
    \vspace{-14pt}
\end{wrapfigure}

As is shown in \autoref{fig: baselines}, we surprisingly observe that Baseline GAL achieves an accuracy as high as $58.63\%$ on CIFAR10. It is worth noting that the training does not require any real training data, but only relies on the knowledge about the vision task (text inputs). It provides a potential solution to train a model from scratch without any data and labeling. This kind of text-to-model capability stems from the zero-shot representation of large language models and the high-fidelity generation of diffusion models, addressing the high cost of labeling, and data collection. Combining the synthesized samples with the labeled samples, \sys significantly improves the performance of SOTA active learning for free (without additional annotation). On CIFAR10, GAL-AL can achieve $92.46\%$ accuracy ($2.9\%$ gap to the fully-supervised learning) with only $10,000$ labels.

\vspace{-2pt}
\subsection{Dataset Reuse and Transferability}
\label{sec: Dataset Reuse}
In this experiment, we train different models on the generated dataset in Section \ref{sec: main exp} to illustrate the data transferability over different models. We focus on four distinct models, which vary in size and architecture: VGG16~\cite{simonyan2014very}, DenseNet121~\cite{huang2017densely}, MobileNetV2~\cite{sandler2018mobilenetv2} and DLA~\cite{yu2018deep}. \autoref{table:transfer} presents the transferability results, showcasing various models' performance when trained on the generated dataset.
The first two rows present the traditional AL results, and the third row is the results of \sys as the same in \autoref{table:whole}.
The following rows show the performance of different models when trained on the data generated for ResNet18.

The data presented in the table reveals that all four models trained on the reused dataset surpass the two baseline AL methods in terms of average model accuracy. 
Notably, with the exception of VGG16, all other models even exceed the baseline active learning methods in model accuracy after each training cycle. 
This indicates that the dataset generated by \sys demonstrates significant reusability and transferability.
Compared to Baseline (Resnet 18), models like DenseNet121, MobileNetV2, and DLA not only demonstrate superior average accuracy but also higher accuracy after each training round. 
This improved performance is likely due to the inherent advantages of their architectural designs. In conclusion, \autoref{table:transfer} demonstrates that the \sys can generate informative data that is not specific to a model but general for a task (or multiple tasks).

\input{tables/transfer}






\subsection{Performance of Pseudo-labeling via Human Annotation}
\label{sec:Performance of Pseudo-labeling via Human Annotation}

\begin{wrapfigure}{r}{0.4\textwidth}
    \centering
    \vspace{-8pt}
    \includegraphics[width=1.0\linewidth]{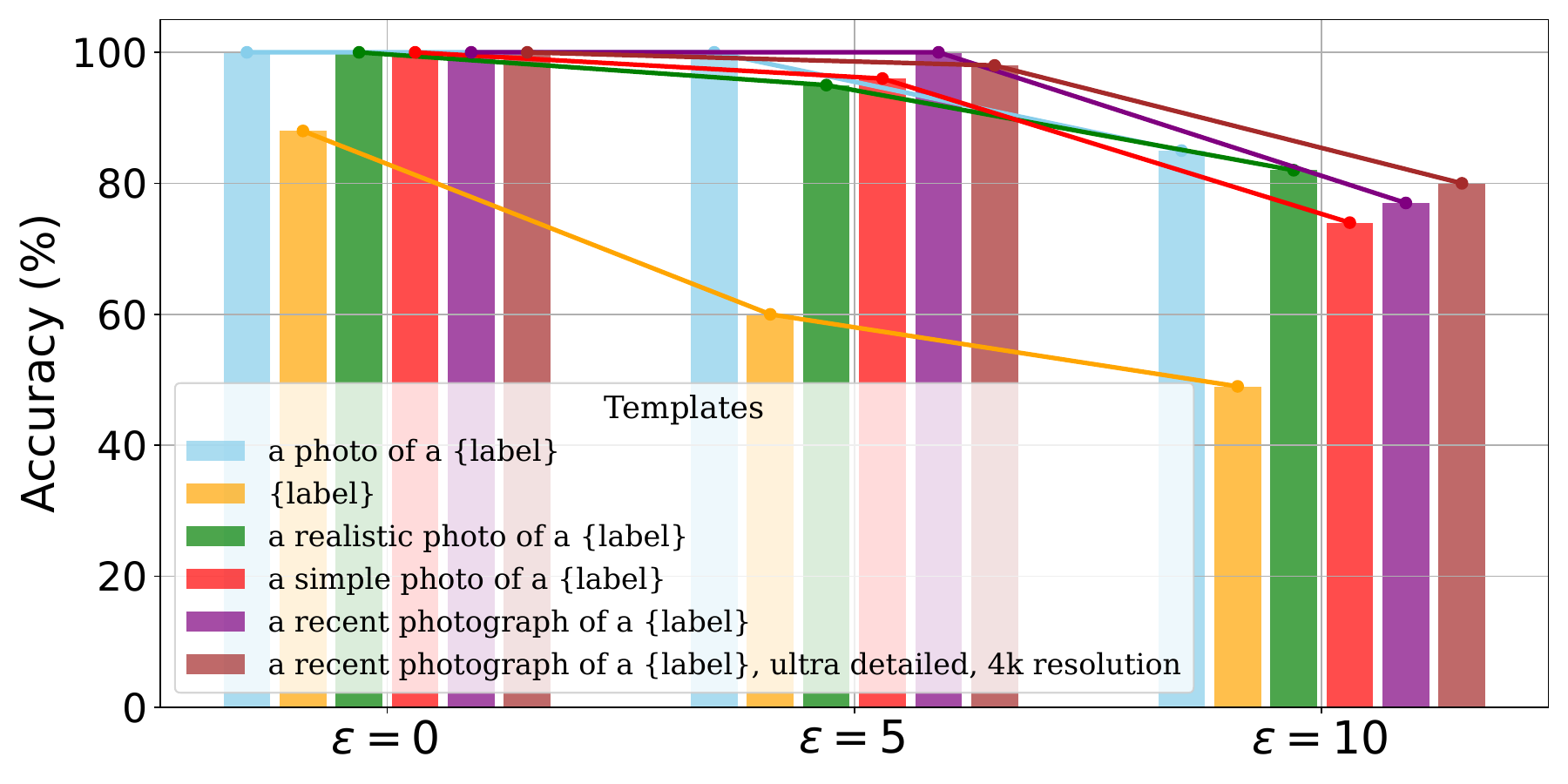}
    \vspace{-14pt}
    \caption{Text-to-image Generation Accuracy (Human Evaluated) vs. $\epsilon$ with Different Templates.}
    \label{fig:histogram_labeling}
    \vspace{-14pt}
\end{wrapfigure}
To evaluate the efficacy of pseudo-labels generated by our model, we compute the accuracy of the pseudo-label through human annotation. Specifically, for each template, we randomly optimize the text embedding under the constraint of various $\epsilon$. The generated image is marked as correct when it accurately represents the pseudo-label. We evaluate the accuracy for $100$ images on each setting. In \autoref{fig:histogram_labeling}, we observe that descriptive templates like ``a photo of a \{label\}'' demonstrated higher robustness, maintaining high accuracy even at higher distortion levels. In contrast, simpler templates, particularly ``{label'', were more susceptible to distortion, showing a marked decrease in accuracy. These results underscore the significance of prompt complexity in maintaining the fidelity of generative models under varying conditions, highlighting the need for well-structured prompts in enhancing model robustness.


\subsection{Ablation Study}
\label{sec: abl exp}


\begin{table*}[!h]
    \centering
   \vspace{-.2in}
    \caption{Comparison of Different Text Templates on CIFAR10.}
    \vspace{-.1in}
    \resizebox{\linewidth}{!}{\begin{tabular}{c c c c c c c}
        \toprule
        Text Template & 1000 &2000 &3000 &4000 &5000 &Average \\
        \hline
        ``a photo of a \{label\}"                                                   & 0.6091 $\pm$ 0.0032 & 0.7388 $\pm$ 0.0051 & 0.8103 $\pm$ 0.0030& 0.8403 $\pm$ 0.0031 & 0.8638 $\pm$ 0.0036 & 0.7642 $\pm$ 0.0085 \\
        ``\{label\}"                                                                & 0.5566 $\pm$ 0.0072 & 0.6969 $\pm$ 0.001& 0.7921 $\pm$ 0.0015& 0.8312 $\pm$ 0.0001 & 0.8525 $\pm$ 0.0005&0.7521 $\pm$ 0.0077\\
        ``a realistic photo of a \{label\}"                                         & 0.6074 $\pm$ 0.0049& 0.7279 $\pm$ 0.0015& 0.8049 $\pm$ 0.0054& 0.8489 $\pm$ 0.0015& 0.8644 $\pm$ 0.0026&0.7642 $\pm$ 0.0039\\
        ``a simple photo of a \{label\}"                                            & 0.6043 $\pm$ 0.0025& 0.7466 $\pm$ 0.0067& 0.8063 $\pm$ 0.0036& 0.8367 $\pm$ 0.0033& 0.8607 $\pm$ 0.0027&\cellcolor{green!20}\textbf{0.7718 $\pm$ 0.0005}\\
        ``a recent color photograph of a \{label\}"                                 & 0.5998 $\pm$ 0.0032& 0.7475 $\pm$ 0.0061& 0.8051 $\pm$ 0.0030& 0.8459 $\pm$ 0.0011&0.8682 $\pm$ 0.0020 &\cellcolor{green!20}\textbf{0.7721 $\pm$ 0.0002}\\
        \footnotesize{"a recent color photograph of a \{label\}, ultra detailed, ..."}   & 0.5842 $\pm$ 0.0062& 0.7361 $\pm$ 0.0009& 0.7986 $\pm$ 0.0057& 0.8328 $\pm$ 0.0032& 0.8667 $\pm$ 0.0035&0.7672 $\pm$ 0.0000\\
        \bottomrule
    \end{tabular}}
    \label{tab:text_template_comparison}
\end{table*}

\noindent\textbf{Text Template Comparisons for CIFAR10}. In this section, we evaluate the performance of \sys versus various different text templates. 
The text input is constructed using different text templates in \autoref{tab:text_template_comparison} by replacing ``\{label\}'' with the label name of each class. In \autoref{tab:text_template_comparison}, it is observed that none of the templates dominate over all annotation budgets. However, the templates "a realistic photo of a \{label\}" and "a simple photo of a \{label\}" generally perform better than other templates, while simply using "\{label\}" as the template results in the worse performance on average. This indicates that using appropriate text templates to generate the image can improve the image generation quality (see Section \ref{sec:Performance of Pseudo-labeling via Human Annotation} for more results) and further improve the learning.

\begin{wraptable}{r}{0.40\linewidth}
\centering
\caption{Ablation study of $\epsilon$.}
\label{tab:epsilon}
\resizebox{0.40\columnwidth}{!}{
\begin{tabular}{l c c c c c c}
\toprule
Method & $\epsilon$ & 1000 & 2000 & 3000 & 4000 & 5000 \\
\toprule
GALOT w/o Opt & 0 & \cellcolor{green!20}\textbf{0.6547} & 0.7700 & 0.8228 & 0.8497 & 0.8733 \\
\hline
GALOT w/ Opt  & 0.25 & 0.6396 & 0.7675 & 0.8275 & 0.8576 & 0.8739 \\
GALOT w/ Opt & 0.50 & 0.6440 & 0.7676 & \cellcolor{green!20}\textbf{0.8278} & \cellcolor{green!20}\textbf{0.8627} & \cellcolor{green!20}\textbf{0.8795} \\
GALOT w/ Opt & 1.00 & 0.6407 & 0.7549 & 0.8213 & 0.8608 & 0.8764 \\
GALOT w/ Opt & 10.00 &0.6224 & \cellcolor{green!20}\textbf{0.7702} & 0.8111 & 0.8471 &0.8655 \\
\bottomrule
\end{tabular}}
\vspace{-6pt}
\end{wraptable}

\noindent\textbf{$\epsilon$ in Text Embedding Optimization}. The impact of varying values of $\epsilon$ on text embedding optimization is illustrated in \autoref{tab:epsilon}. The results show that using $\epsilon>0$ can achieve better performance than setting $\epsilon=0$, except for the $|L|=1,000$ setting. This indicates that the text embedding optimization is effective.
\label{sec: abl acq}

\begin{wrapfigure}{r}{0.3\linewidth}
    \vspace{-14pt}
    \includegraphics[width=\linewidth]{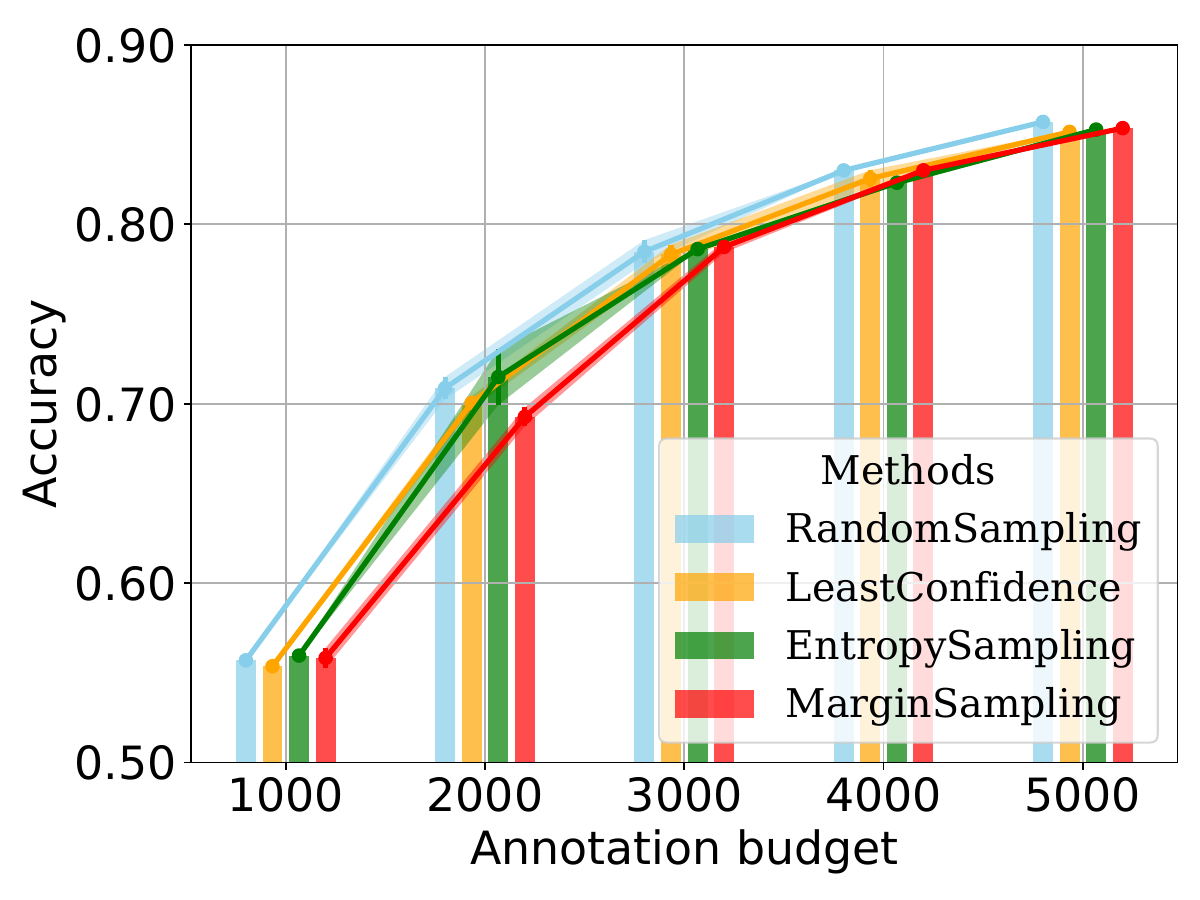}\vspace{-12pt}
    \caption{Accuracy vs Acquisition Methods $\sigma_{GAL}$.}
    \label{fig:ablation_methods}
    \vspace{-18pt}
\end{wrapfigure}


\noindent\textbf{Acquisition Function $\sigma_{GAL}$}. In this study, we assessed the effectiveness of four acquisition functions for GAL in \sys: random sampling, entropy sampling, margin sampling, and least confidence. 
According to the results presented in~\autoref{fig:ablation_methods}, random sampling demonstrates superior performance, particularly in the last three cycles.
While this might seem unconventional, a similar trend has been observed on TinyImageNet in \autoref{table:whole}. It is probably when the model is not well-trained, that the guidance of the AL acquisition function is meaningless, since when the model predicts more accurately (see CIFAR10 and CIFAR100 in \autoref{table:whole}), other acquisition function gradually outperforms the random sampling.




\begin{wrapfigure}{r}{0.3\linewidth}
    \vspace{-14pt}
    \includegraphics[width=\linewidth]{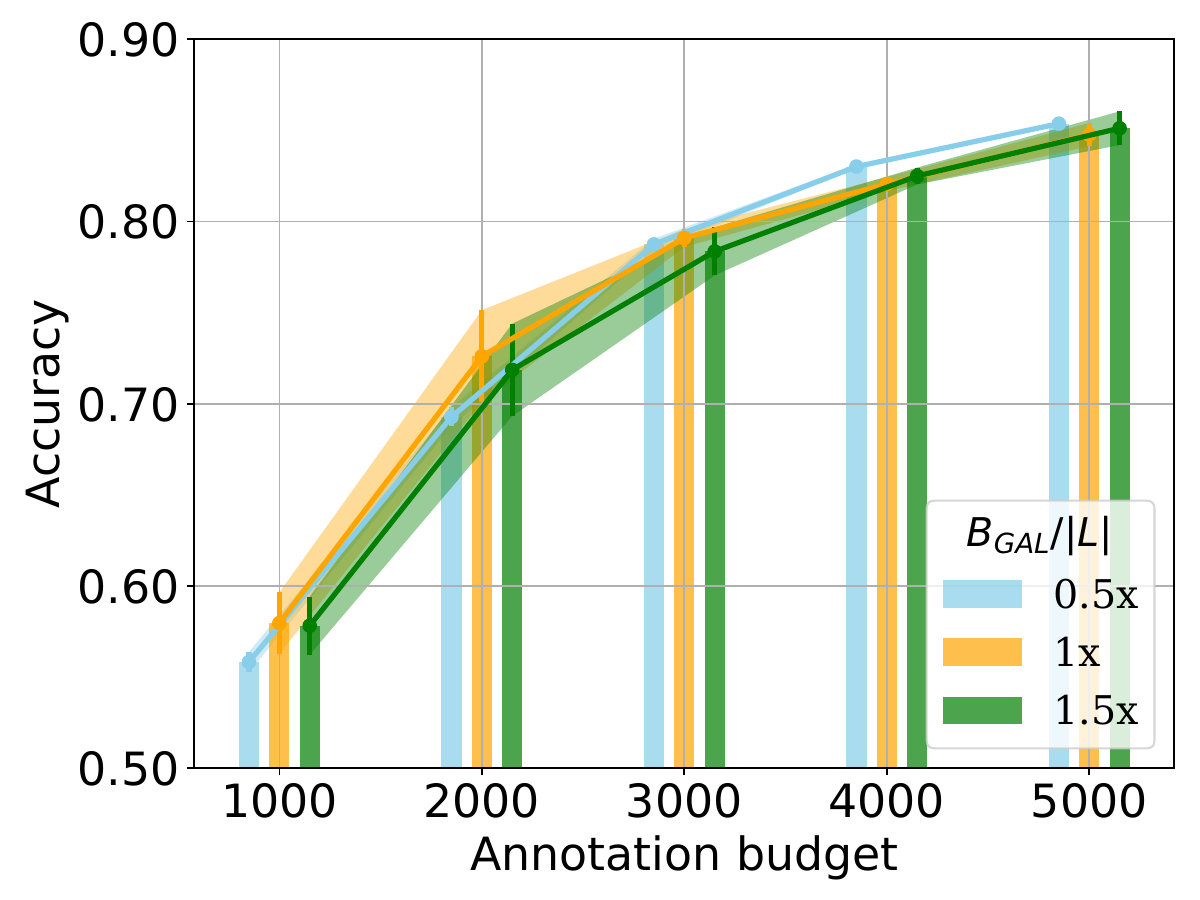}
    \vspace{-18pt}
    \caption{Accuracy vs $B_{GAL}$.}
    \label{fig:ablation_GAL}
    \vspace{-24pt}
\end{wrapfigure}

\vspace{0.05in}
\noindent\textbf{Sampling Number $B_{GAL}$ vs. Annotation Budgets}. In this experiment, we examine the effects of the number of generated samples on the performance of \sys. We generate $0.5|L|$, $|L|$, and $1.5|L|$ data samples for \sys.
The results are presented in~\autoref{fig:ablation_GAL}. The performance with different $B_{GAL}$ does not vary a lot because, in the same cycle, all the $B_{GAL}$ samples are generated with the same text embedding, thus cannot provide new data distribution, resulting in a similar performance. 

\begin{table}[ht]
\centering
\begin{minipage}{0.48\textwidth}
    \centering
    \caption{Training Scale Comparison}
    \label{tab:training scale}
    \resizebox{\textwidth}{!}{%
    \begin{tabular}{lc c c c c}
    \toprule
    Methods & 1000 & 2000 & 3000 & 4000 & 5000 \\ \midrule
    MarginSampling (300 epochs) & 0.5578 & 0.6686 & 0.7369 & 0.8001 & 0.8383 \\
    MarginSampling (600 epochs) & 0.5730 & 0.6810 & 0.7793 & 0.8356 & 0.8529 \\
    GALOT (300 epochs) & \cellcolor{green!20}\textbf{0.6460} & \cellcolor{green!20}\textbf{0.7551} & \cellcolor{green!20}\textbf{0.8245} & \cellcolor{green!20}\textbf{0.8594} & \cellcolor{green!20}\textbf{0.8763} \\
    \bottomrule
    \end{tabular}%
    }
\end{minipage}\hfill
\begin{minipage}{0.48\textwidth}
    \centering
    \caption{Extra Generation}
    \label{tab:extra gen}
    \resizebox{\textwidth}{!}{%
    \begin{tabular}{lc c c c c}
    \toprule
    Method & 1000 & 2000 & 3000 & 4000 & 5000 \\ \midrule
    Gen \textbar L\textbar / Sample \textbar L\textbar & 0.6460 & 0.7551 & 0.8245 & \cellcolor{green!20}\textbf{0.8594} & 0.8763 \\
    Gen 2\textbar L\textbar / Sample \textbar L\textbar & \cellcolor{green!20}\textbf{0.6504} & \cellcolor{green!20}\textbf{0.7563} & \cellcolor{green!20}\textbf{0.8279} & 0.8517 & \cellcolor{green!20}\textbf{0.8764} \\
    \bottomrule
    \end{tabular}%
    }
\end{minipage}
\end{table}

\noindent\textbf{Training Scale Comparison}. Since \sys uses additional synthesized datasets during the training, compared to the standard training, \sys may introduce more training iterations, although the annotation budget is the same. This brings potential uncertainty to the validation of the improvement. Therefore, to explicitly quantify the improvement caused by the better data distribution, we train the AL baseline with the same iteration as ours. \autoref{tab:training scale} shows that increasing the training iteration of AL baselines improves the performance slightly, but \sys still drastically outperforms the AL baseline.

\noindent\textbf{Extra Generation}. \sys can generate as many as possible samples if the computational resource allows, therefore, we also evaluate the performance when generating $2|L|$ samples while randomly selecting $|L|$ samples. \autoref{tab:extra gen} shows that generating more samples and randomly selecting half of them can lead to slightly performance improvement.


\begin{wraptable}{r}{0.5\linewidth}
\vspace{-12pt}
\centering
\caption{Different $\sigma_{AL}$ and $\sigma_{GAL}$.}
\label{tab:various acqs}
\resizebox{\linewidth}{!}{
\begin{tabular}{lc c c c c}
\toprule
Strategy         & 1000  & 2000  & 3000  & 4000  & 5000  \\ \hline
MarginSampling   & \cellcolor{green!20}\textbf{0.6460} & 0.7551 & \cellcolor{green!20}\textbf{0.8245} & \cellcolor{green!20}\textbf{0.8594} & 0.8763 \\ 
RandomSampling   & 0.6297 & \cellcolor{green!20}\textbf{0.7615} & 0.8071 & 0.8382 & 0.8580 \\ 
EntropySampling  & 0.6358 & 0.7490 & 0.8212 & 0.8581 & \cellcolor{green!20}\textbf{0.8811} \\ \bottomrule
\end{tabular}}
\vspace{-18pt}
\end{wraptable}

\noindent\textbf{Different $\sigma_{AL}$ and $\sigma_{GAL}$}. We also test various $\sigma_{AL}$ and $\sigma_{GAL}$ on \sys. This is different from the ablation study of varying $\sigma_{GAL}$, where the acquisition function for the AL baseline is fixed on margin sampling. From \autoref{tab:various acqs}, it is observed that when using margin sampling for both the $\sigma_{AL}$ and $\sigma_{GAL}$, the performance is better.

%% file: tables/whole_compare.tex
\begin{table*}[t]
	\centering
	\scriptsize
	\caption{Comparison with SOTA AL. We present the accuracy on the test set vs. different annotation budget $|L|$. The first column 
 presents the dataset and the fully-supervised accuracy. The last column presents the average accuracy over all settings. }\label{table:whole}

\resizebox{\linewidth}{!}{

\begin{tabular}{@{}llrrrrrrrrrrr@{}}
\toprule
Dataset &
  Method &
  1000 &
  2000 &
  3000 &
  4000 &
  5000 &
  6000 &
  7000 &
  8000 &
  9000 &
  10000 &
  Average \\ \midrule
\multirow{13}{*}{\begin{tabular}[c]{c@{}l@{}}CIFAR-10\\ (0.9536)\end{tabular}} &
  GAAL \cite{GAAL} &
  0.2811 & 
  0.2132 &
  0.2086 &
  0.2140 &
  0.2086 &
  0.1991 &
  0.2005 & 
  0.2012 &
  0.1981 &
  0.2011 &
  0.2126
  \\
  &
  RandomSamping &
  0.5498 &
  0.6631 &
  0.7347 &
  0.7912 &
  0.8195 &
  0.8369 &
  0.8513 &
  0.8645 &
  0.8789 &
  0.8841 &
  0.7874 \\
 &
  EntropySampling \cite{ALSurvey} &
  0.5466 &
  0.6718 &
  0.7797 &
  0.8109 &
  0.8460 &
  0.8580 &
  \underline{0.8876} &
  0.8901 &
  0.9013 &
  0.9063 &
  0.8098 \\
 &
  MarginSampling \cite{MarginSampling} &
  0.5506 &
  0.6800 &
  \underline{0.8127} &
  0.8127 &
  0.8400 &
  0.8546 &
  0.8830 &
  0.8927 &
  \underline{0.9058} &
  0.9049 &
  \underline{0.8137} \\
 &
  LeastConfidence \cite{LeastConfidence} &
  0.5591 &
  0.6790 &
  0.7661 &
  0.8111 &
  \underline{0.8481} &
  0.8640 &
  0.8790 &
  0.8932 &
  0.9035 &
  \underline{0.9145} &
  0.8118 \\
 &
  BALD \cite{BALD} &
  0.5589 &
  0.6591 &
  0.7612 &
  \underline{0.8184} &
  \underline{0.8481} &
  0.8654 &
  0.8788 &
  0.8897 &
  0.9019 &
  0.9077 &
  0.8089 \\
 &
  BADGE \cite{BADGE} &
  0.5592 &
  \underline{0.6882} &
  0.7698 &
  0.8160 &
  0.8458 &
  0.8582 &
  0.8774 &
  0.8877 &
  0.9006 &
  0.9049 &
  0.8108 \\
 &
  MeanSTD \cite{MeanSTD} &
  0.5350 &
  0.7046 &
  0.7393 &
  0.8116 &
  0.8398 &
  0.8679 &
  0.8828 &
  0.8909 &
  0.8971 &
  0.9093 &
  0.8078 \\
 &
  VarRatio \cite{VarRatio} &
  0.5534 &
  0.6616 &
  0.7733 &
  0.8045 &
  0.8413 &
  \underline{0.8682} &
  0.8813 &
  \underline{0.8995} &
  0.8994 &
  0.9079 &
  0.8090 \\
 &
  KMeans &
  \underline{0.5594} &
  0.6646 &
  0.7381 &
  0.7807 &
  0.8078 &
  0.8347 &
  0.8516 &
  0.8623 &
  0.8743 &
  0.8800 &
  0.7854 \\
 &
  CoreSet \cite{CoreSet} &
  0.5529 &
  0.6748 &
  0.7477 &
  0.8056 &
  0.8361 &
  0.8534 &
  0.8739 &
  0.8866 &
  0.8951 &
  0.9024 &
  0.8029 \\
 &
  LossPrediction \cite{LPL} &
  0.5291 &
  0.6319 &
  0.7205 &
  0.7366 &
  0.7688 &
  0.8252 &
  0.8523 &
  0.7977 &
  0.8786 &
  0.8829 &
  0.7624 \\ \cmidrule(l){2-13} 
  &
   \multicolumn{1}{l}{\sys (basic)} &
   0.5874 &
   0.7162 &
   0.7966 &
   0.8288 &
   0.8631 &
   0.8790 & 
   0.8937 &
   0.9012 &
   0.9133 &
   0.9179 &
   0.8297\\
 &
  \multicolumn{1}{l}{\sys} &
  \cellcolor{green!20}\textbf{0.6460} &
  \cellcolor{green!20}\textbf{0.7551} &
  \cellcolor{green!20}\textbf{0.8245} &
  \cellcolor{green!20}\textbf{0.8594} &
  \cellcolor{green!20}\textbf{0.8763} &
  \cellcolor{green!20}\textbf{0.8958} &
  \cellcolor{green!20}\textbf{0.9073} &
  \cellcolor{green!20}\textbf{0.9130} &
  \cellcolor{green!20}\textbf{0.9188} &
  \cellcolor{green!20}\textbf{0.9246} &
  \cellcolor{green!20}\textbf{0.8521} \\
 &
  Improvement &
  +0.0878 &
  +0.0505 &
  +0.0118 &
  +0.0410 &
  +0.0282 &
  +0.0276 &
  +0.0197 &
  +0.0135 &
  +0.0130 &
  +0.0101 &
  +0.0384 \\ \midrule
\multirow{13}{*}{\begin{tabular}[c]{c@{}l@{}}CIFAR-100\\ (0.7668)\end{tabular}} &
  RandomSamping &
  0.1408 &
  0.2080 &
  0.2656 &
  0.3086 &
  0.3635 &
  0.3995 &
  0.4418 &
  0.4703 &
  0.4973 &
  0.5250 &
  0.3620 \\
 &
  EntropySampling \cite{ALSurvey} &
  0.1442 &
  0.1938 &
  0.2474 &
  0.2957 &
  0.3517 &
  0.3851 &
  0.4161 &
  0.4441 &
  0.4839 &
  0.5379 &
  0.3500 \\
 &
  MarginSampling \cite{MarginSampling} &
  \underline{0.1482} &
  0.1985 &
  0.2528 &
  0.3155 &
  0.3788 &
  0.4056 &
  \underline{0.4476} &
  0.4951 &
  \underline{0.5237} &
  0.5343 &
  0.3700 \\
 &
  LeastConfidence \cite{LeastConfidence} &
  0.1448 &
  0.1938 &
  0.2521 &
  0.3087 &
  0.3553 &
  0.3830 &
  0.4216 &
  0.4545 &
  0.4981 &
  0.5415 &
  0.3553 \\
 &
  BALD \cite{BALD} &
  0.1450 &
  0.2012 &
  0.2644 &
  0.3099 &
  0.3553 &
  0.3960 &
  0.4245 &
  0.4634 &
  0.4880 &
  0.5274 &
  0.3575 \\
 &
  BADGE \cite{BADGE} &
  0.1441 &
  0.2066 &
  0.2644 &
  0.3137 &
  \underline{0.3857} &
  \underline{0.4164} &
  0.4462 &
  0.4810 &
  0.5094 &
  \underline{0.5507} &
  \underline{0.3718} \\
 &
  MeanSTD \cite{MeanSTD} &
  0.1434 &
  0.2016 &
  0.2673 &
  0.3221 &
  0.3560 &
  0.3877 &
  0.4298 &
  0.4573 &
  0.4837 &
  0.5270 &
  0.3576 \\
 &
  VarRatio \cite{VarRatio} &
  0.1447 &
  0.1954 &
  0.2463 &
  0.3026 &
  0.3480 &
  0.3807 &
  0.4316 &
  0.4725 &
  0.4829 &
  0.5415 &
  0.3546 \\
 &
  KMeans &
  0.1462 &
  \underline{0.2137} &
  \underline{0.2707} &
  \underline{0.3261} &
  0.3631 &
  0.3863 &
  0.4165 &
  0.4521 &
  0.4833 &
  0.5052 &
  0.3563 \\
 &
  CoreSet \cite{CoreSet} &
  0.1424 &
  0.2059 &
  0.2560 &
  0.3169 &
  0.3646 &
  0.3988 &
  0.4413 &
  \underline{0.4972} &
  0.5164 &
  0.5483 &
  0.3688 \\
 &
  LossPrediction \cite{LPL} &
  0.1322 &
  0.1780 &
  0.2387 &
  0.2988 &
  0.3275 &
  0.3507 &
  0.4143 &
  0.4126 &
  0.4444 &
  0.4566 &
  0.3254 \\ \cmidrule(l){2-13} 
   &
  \multicolumn{1}{l}{\sys (basic)} &
  0.1669 &
  0.2713 &
  0.3310 &
  0.3885 &
  0.4474 &
  0.4777 &
  0.5105 &
  0.5411 &
  0.5629 &
  0.5943 &
  0.4292 \\
 &
  \multicolumn{1}{l}{\sys} &
  \cellcolor{green!20}\textbf{0.1892} &
  \cellcolor{green!20}\textbf{0.2893} &
  \cellcolor{green!20}\textbf{0.3585} &
  \cellcolor{green!20}\textbf{0.4170} &
  \cellcolor{green!20}\textbf{0.4683} &
  \cellcolor{green!20}\textbf{0.4921} &
  \cellcolor{green!20}\textbf{0.5346} &
  \cellcolor{green!20}\textbf{0.5619} &
  \cellcolor{green!20}\textbf{0.5807} &
  \cellcolor{green!20}\textbf{0.5974} &
  \cellcolor{green!20}\textbf{0.4489} \\
 &
  Improvement &
   +0.0410&
   +0.0756&
   +0.0878&
   +0.0909&
   +0.0826&
   +0.0757&
   +0.0870&
   +0.0647&
   +0.0570&
   +0.0467&
   +0.0771\\ \midrule
\multirow{13}{*}{\begin{tabular}[c]{c@{}l@{}}TinyImageNet\\ (0.5054)\end{tabular}} &
  RandomSamping &
  0.1279 &
  0.1944 &
  0.2327 &
  0.2617 &
  \underline{0.2574} &
  0.3064 &
  \underline{0.3348} &
  0.3477 &
  0.3604 &
  \underline{0.3613} &
  0.2785 \\
 &
  EntropySampling \cite{ALSurvey} &
  0.1256 &
  0.1809 &
  0.2152 &
  0.2399 &
  0.2222 &
  0.2817 &
  0.3159 &
  0.3149 &
  0.3378 &
  0.3313 &
  0.2565 \\
 &
  MarginSampling \cite{MarginSampling} &
  0.1274 &
  0.1764 &
  0.2290 &
  0.2496 &
  0.2350 &
  0.3029 &
  0.3137 &
  0.3355 &
  0.3517 &
  0.3472 &
  0.2668 \\
 &
  LeastConfidence \cite{LeastConfidence} &
  0.1269 &
  0.1799 &
  0.2187 &
  0.2427 &
  0.2223 &
  0.2909 &
  0.3031 &
  0.3152 &
  0.3351 &
  0.3291 &
  0.2564 \\
 &
  BALD \cite{BALD} &
  0.1282 &
  0.1877 &
  0.2230 &
  0.251 &
  0.2466 &
  0.2942 &
  0.3142 &
  0.3263 &
  0.3364 &
  0.3387 &
  0.2646 \\
 &
  BADGE \cite{BADGE} &
  0.1278 &
  0.1858 &
  \underline{0.2350} &
  0.2651 &
  0.2467 &
  0.3073 &
  0.3190 &
  0.3410 &
  0.3562 &
  0.3518 &
  0.2736 \\
 &
  MeanSTD \cite{MeanSTD} &
  0.1253 &
  0.1815 &
  0.2250 &
  0.2447 &
  0.2422 &
  0.2920 &
  0.3123 &
  0.3279 &
  0.3320 &
  0.3351 &
  0.2618 \\
 &
  VarRatio \cite{VarRatio} &
  \underline{0.1288} &
  0.1736 &
  0.2025 &
  0.2429 &
  0.2222 &
  0.2915 &
  0.2997 &
  0.3061 &
  0.3280 &
  0.3279 &
  0.2523 \\
 &
  KMeans &
  0.1259 &
  \underline{0.1965} &
  0.2329 &
  \underline{0.2704} &
  0.2546 &
  \underline{0.3153} &
  0.3329 &
  \underline{0.3510} &
  \underline{0.3633} &
  0.3552 &
  \underline{0.2798} \\
 &
  CoreSet \cite{CoreSet} &
  0.1274 &
  0.1793 &
  0.2114 &
  0.2441 &
  0.2258 &
  0.2944 &
  0.3092 &
  0.3209 &
  0.3358 &
  0.3348 &
  0.2583 \\
 &
  LossPrediction \cite{LPL} &
  0.0424 &
  0.0599 &
  0.0778 &
  0.1113 &
  0.1136 &
  0.1427 &
  0.1500 &
  0.1654 &
  0.1800 &
  0.1775 &
  0.1221 \\ \cmidrule(l){2-13}
  &
  \multicolumn{1}{l}{\sys (basic)} &
  0.1660 &
  \cellcolor{green!20}\textbf{0.2332} &
  0.2805 &
  0.3087 &
  0.3134 &
  \cellcolor{green!20}\textbf{0.3540} &
  \cellcolor{green!20}\textbf{0.3566} &
  \cellcolor{green!20}\textbf{0.3683} &
  \cellcolor{green!20}\textbf{0.3730} &
  \cellcolor{green!20}\textbf{0.3790} &
   \cellcolor{green!20}\textbf{0.3133} \\
 &
  \multicolumn{1}{l}{\sys} &
  \cellcolor{green!20}\textbf{0.1681} &
  0.2316 &
  \cellcolor{green!20}\textbf{0.2812} &
  \cellcolor{green!20}\textbf{0.3109} &
  \cellcolor{green!20}\textbf{0.3158} &
  0.3481 &
  0.3562 &
  0.3665 &
  0.3705 &
  0.3783 &
  0.3127 \\
 &
  Improvement &
  +0.0393 &
  +0.0367 &
  +0.0462 &
  +0.0405 &
  +0.0584 &
  +0.0387 &
  +0.0218 &
  +0.0173 &
  +0.0097 &
  +0.0177 &
  +0.0335 \\ \bottomrule
\end{tabular}
}
\vspace{-12pt}
\end{table*}

%% file: tables/transfer.tex
\begin{table}[t]
	\centering
	\scriptsize
	\caption{Data Reuse for Different Networks on CIFAR10. }\label{table:transfer}
\resizebox{\columnwidth}{!}{
\begin{tabular}{@{}lrrrrrrrrrrr@{}}
\toprule
\textbf{Model} & 1000   & 2000   & 3000   & 4000   & 5000   & 6000   & 7000   & 8000   & 9000   & 10000  & Average \\ \midrule
MarginSampling       & 0.5506 & 0.6800 & 0.8127 & 0.8127 & 0.8400 & 0.8546 & 0.8830 & 0.8927 & 0.9058 & 0.9049 & 0.8137 \\
LeastConfidence      & 0.5591 & 0.6790 & 0.7661 & 0.8111 & 0.8481 & 0.8640 & 0.8790 & 0.8932 & 0.9035 & 0.9145 & 0.8118 \\ \midrule
Baseline (Resnet 18) & 0.6460 & 0.7551 & 0.8245 & 0.8594 & 0.8763 & 0.8958 & 0.9073 & 0.9130 & 0.9188 & 0.9246 & 0.8521 \\ \midrule
VGG16                      & 0.6597 & 0.7687 & 0.8069 & 0.8362 & 0.8587 & 0.8746 & 0.8834 & 0.8903 & 0.8915 & 0.9029 & 0.8373  \\
DenseNet121                & 0.6571 & 0.7832 & 0.8434 & 0.8569 & 0.8872 & 0.8969 & 0.9143 & 0.9199 & 0.9276 & 0.9359 & 0.8622  \\
MobileNetv2                & 0.6551 & 0.7689 & 0.8137 & 0.8491 & 0.8729 & 0.8881 & 0.8996 & 0.9093 & 0.9196 & 0.9214 & 0.8498  \\
DLA                        & 0.6634 & 0.7747 & 0.8161 & 0.8546 & 0.8731 & 0.8873 & 0.8986 & 0.9080 & 0.9184 & 0.9252 & 0.8519  \\
\bottomrule
\end{tabular}


}
\vspace{-12pt}
\end{table}

%% file: sec/Conclusion.tex
Our innovative framework GALOT successfully merges AL with T2I synthesis, overcoming traditional AL's data limitations. By optimizing text embeddings to generate informative samples, GALOT broadens data diversity and training efficiency, pioneering the text-to-model approaches. Our evaluations demonstrate its superiority over existing AL methods, marking a significant step forward in developing generative active learning with zero-shot T2I generation. 

%% file: sec/Appendix.tex
\section{Detailed Experimental Settings}

\label{apd: exp setting}
\begin{table*}[!b]
    \centering
    \caption{Detailed Experimental Settings}
    \resizebox{\linewidth}{!}{\begin{tabular}{c|c c c c c c c c c c c}
    \hline
         Experiment & Text Template & $\sigma_{GAL}$ & $\sigma_{AL}$ & $N$ & $\epsilon$ & $\alpha$ & $k$ & $n$ & $B_{AL}$ & $B_{GAL}$ & Total Annotation\\
    \hline
         \autoref{table:whole} & {\footnotesize a realistic photo of a \{label\}} & MarginSampling &MarginSampling & 10 & linear$(0,0.5)$ &$\epsilon/5$ & 6 & 10 &1,000& $|L|$ & 10,000 \\
          \autoref{table:whole} (basic) & {\footnotesize \{label\}} & MarginSampling &MarginSampling & 10 & 0 &$\epsilon/5$ & 6 & 10 &1,000& $|L|$ & 10,000 \\
         \autoref{fig: baselines} & {\footnotesize a realistic photo of a \{label\}} & MarginSampling &MarginSampling & 10 & linear$(0,0.5)$ &$\epsilon/5$ & 6 & 10 &1,000& $|L|$ & 10,000 \\
         \autoref{table:transfer} & {\footnotesize a realistic photo of a \{label\}} & MarginSampling &MarginSampling & 10 & linear$(0,0.5)$ &$\epsilon/5$ & 6 & 10 &1,000& $|L|$ & 10,000 \\
         \autoref{fig:visual_comparisons} & \{label\} & MarginSampling &MarginSampling & 1 &linear$(0,0.5)$ & $\epsilon/5$ & 6 & 10 & $5,000$ & $2,500$ & $5,000$ \\
         \autoref{fig:histogram_labeling} & -- & RandomSampling & None & 10 & -- &$\epsilon/5$ & 6 & 10 & 0 & 100 & 0 \\
        \autoref{tab:text_template_comparison} & -- & MarginSampling &MarginSampling & 5 & linear$(0,0.5)$ &$\epsilon/5$ & 6 & 10 &1,000& $0.5|L|$ & 5,000 \\
         \autoref{tab:epsilon} & a realistic photo of a \{label\} & MarginSampling &MarginSampling & 5 & -- &$\epsilon/5$ & 6 & 10 &1,000& $|L|$ & 5,000 \\
         \autoref{fig:ablation_methods} & \{label\} & -- &MarginSampling & 5 & linear$(0,0.5)$ &$\epsilon/5$ & 6 & 10 &1,000& $0.5|L|$ & 5,000\\
         \autoref{fig:ablation_GAL} &  \{label\} & MarginSampling &MarginSampling & 5 & linear$(0,0.5)$ &$\epsilon/5$ & 6 & 10 &1,000& -- & 5,000\\
        \autoref{tab:training scale} & a realistic photo of a \{label\} & MarginSampling &MarginSampling & 5 & linear$(0,0.5)$ &$\epsilon/5$ & 6 & 10 &1,000& $|L|$ & 5,000\\
         \autoref{tab:extra gen} & a realistic photo of a \{label\} & MarginSampling &MarginSampling & 5 & linear$(0,0.5)$ &$\epsilon/5$ & 6 & 10 &1,000& -- & 5,000\\
         \autoref{tab:various acqs} & a realistic photo of a \{label\} & -- &-- & 5 & linear$(0,0.5)$ &$\epsilon/5$ & 6 & 10 &1,000& $|L|$ & 5,000\\
    \hline
    \end{tabular}}

    \label{tab:exp setting}
\end{table*}

\begin{table*}[!ht]
    \centering
    \caption{Detailed Human Annotation Results (number of accurate samples in 10 generated samples)}

    \resizebox{\linewidth}{!}{\begin{tabular}{c| c c c c c c c c c c c c}
    \toprule
         &Template & airplane & automobile & bird & cat & deer & dog & frog & horse & ship & truck & Total \\
    \hline
                        &``a photo of a \{label \}'' & 10& 10& 10& 10& 10& 10& 10& 10& 10& 10& 100 \\
                        &``\{label \}'' & 10& 8& 9& 6& 10& 10& 10& 10& 5& 10& 88 \\
    $\epsilon=0$        &``a realistic photo of a \{label \}'' & 10& 10& 10& 10& 10& 10& 10& 10& 10& 10& 100 \\
                        &``a simple photo of a \{label \}'' & 10& 10& 10& 10& 10& 10& 10& 10& 10& 10& 100 \\
                        &``a recent photograph of a \{label \}'' & 10& 10& 10& 10& 10& 10& 10& 10& 10& 10& 100 \\
                        &{\footnotesize ``a recent photograph of a \{label \}, ultra detailed, ...''} & 10& 10& 10& 10& 10& 10& 10& 10& 10& 10& 100 \\
    \hline
                        &``a photo of a \{label \}'' & 10& 10& 10& 10& 10& 10& 10& 10& 10& 10& 100 \\
                        &``\{label \}'' & 10& 1& 0& 3& 10& 0& 10& 10& 6& 10& 60 \\
    $\epsilon=5$        &``a realistic photo of a \{label \}'' & 10& 10& 10& 10& 10& 7& 9& 10& 9& 10& 95 \\
                        &``a simple photo of a \{label \}'' & 7& 10& 10& 10& 10& 10& 10& 9& 10& 10& 96 \\
                        &``a recent photograph of a \{label \}'' & 10& 10& 10& 10& 10& 10& 10& 10& 10& 10& 100 \\
                        &{\footnotesize ``a recent photograph of a \{label \}, ultra detailed, ...''} & 10& 10& 10& 10& 10& 10& 10& 10& 8& 10& 98 \\
    \hline
                        &``a photo of a \{label \}'' & 9& 9& 10& 1& 10& 10& 6& 10& 10& 10& 85 \\
                        &``\{label \}'' & 8& 1& 3& 3& 2& 2& 10& 10& 0& 10& 49 \\
    $\epsilon=10$        &``a realistic photo of a \{label \}'' & 8& 6& 10& 0& 10& 10& 9& 10& 9& 10& 82 \\
                        &``a simple photo of a \{label \}'' & 9& 4& 2& 10& 4& 10& 10& 6& 9& 10& 74 \\
                        &``a recent photograph of a \{label \}'' &9& 10& 5& 7& 10& 10& 10& 2& 4&  10& 77 \\
                        &{\footnotesize ``a recent photograph of a \{label \}, ultra detailed, ...''}&3 & 1& 10& 10& 10& 9& 9& 9& 9& 10&  80 \\
    \bottomrule
         
    \end{tabular}}
    \label{tab:annotation}
\end{table*}

We present the detailed experimental parameter settings in \autoref{tab:exp setting}, where $\epsilon$ denotes the constrain of regularization norm for text embedding deviation, $\alpha$ denotes the step size for updating the text embedding, $k$ denotes the number of sampled images in computing the gradients during updating text embeddings, $n$ denotes the step for updating the text embedding, $B_{GAL}$ denotes the sampling number in GAL, and $B_{AL}$ denotes the sampling number in AL.

In \autoref{table:whole}, GAAL is originally limited to binary classification on MNIST and CIFAR10 with SVM, we extend it to multi-class classification.

In Section \ref{sec: joint sampling}, for each dataset, the AL baseline selects $1,000$ samples per cycle from the training dataset using Margin Sampling. The GAL baseline generates $|L|$ samples per cycle using Margin Sampling, e.g., $1,000$ samples for the first cycle, $2,000$ samples for the second cycle, ..., and $10,000$ samples for the $10$-th cycle. The AL+GAL baseline selects $1,000$ samples and then generates $|L|$ samples per cycle using Margin Sampling. We run $10$ cycles for these baselines. In total, we have $10,000$ samples for the AL and GAL baselines and $20,000$ samples for the AL+GAL baseline. The fully-supervised baseline takes the whole training set of each dataset as training data.

In Section \ref{sec: Dataset Reuse}, we use the generated data set from Section \ref{sec: main exp} based on ResNet18 as the generated data set for other networks. Specifically, for VGG16, DenseNet121, MobileNetV2, and DLA, we run 10 cycles of active learning with the pre-generated data set. In each cycle, $1,000$ samples are selected from the original training set of CIFAR10, and $|L|$ pre-generated samples are also added to the training data.

In Section \ref{sec:Performance of Pseudo-labeling via Human Annotation}, we use random sampling to randomly optimize the text embedding with different constraints of $\epsilon$. For each setting, we generate 10 images for each class in CIFAR10 and then evaluate the accuracy of the generation. The detailed results for each class are also presented in \autoref{tab:annotation}

In Section \ref{sec: abl exp}, we evaluate the effectiveness of different components of GALOT using the default hyper-parameters as in \autoref{tab:exp setting}. In the text template comparison, the text template is changed across different experiments. In the ablation study of $\epsilon$, the $\epsilon$ is set to different values to test the effectiveness of the text embedding optimization.  In the ablation study of various $B_{GAL}$, we change the generation number w.r.t. the cumulative labeled number $|L|$, by a factor of $0.5$, $1.0$, and $1.5$. 

Other experiment settings are also summarized in \autoref{tab:exp setting}. The uncertainty presented in \autoref{tab:text_template_comparison}, \autoref{fig:ablation_methods}, and \autoref{fig:ablation_GAL} is the Standard Error of the Mean (SEM) obtained over 2 runs.

\begin{figure}[!ht]
    \centering
    \includegraphics[width=\linewidth]{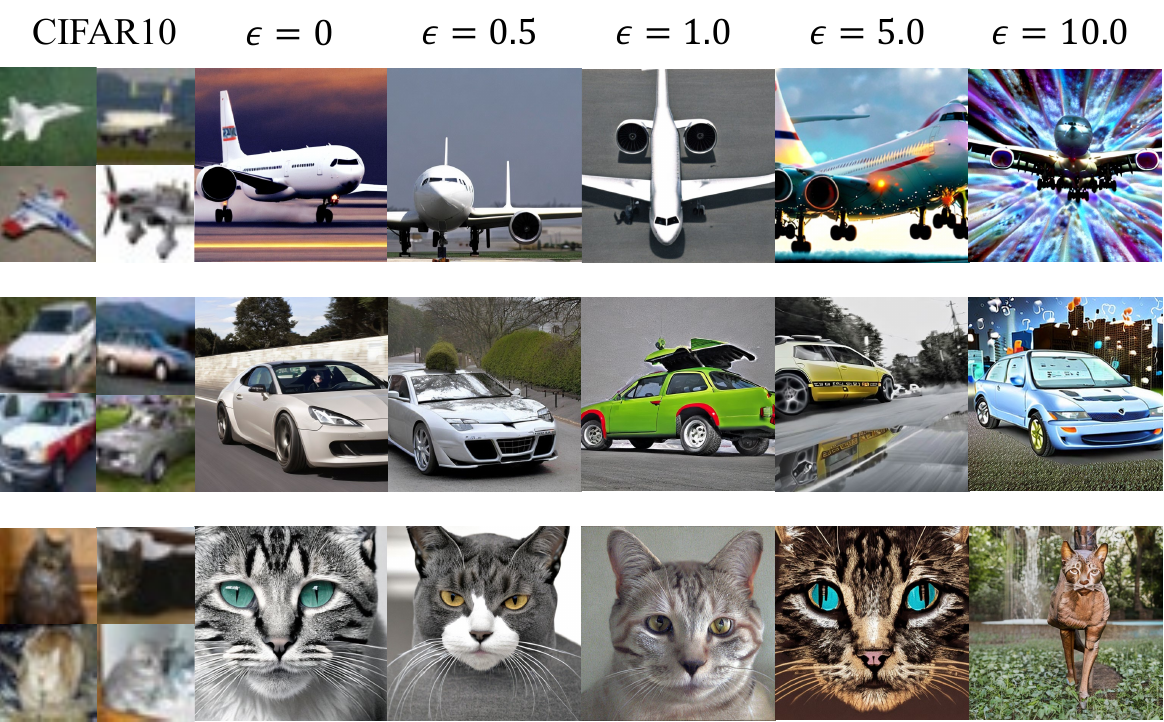}
    \caption{Visual Comparison of Generated Images VS. Real Images on CIFAR10.}
    \label{fig:visual_comparisons}
\end{figure}

\section{T2I Generation Examples}
\label{sec:visual examples}

\begin{figure}[!ht]
    \centering
    \includegraphics[width=\linewidth]{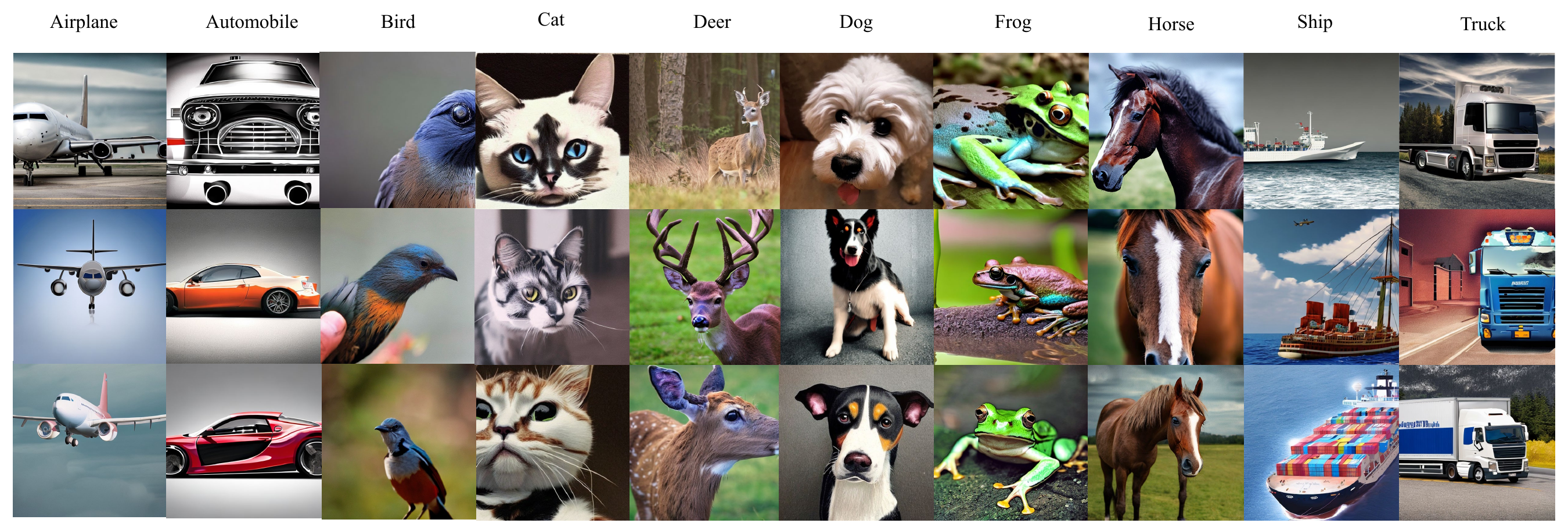}
    \caption{Generated Images for CIFAR10 Classification}
    \label{fig:example1}
\end{figure}

\subsubsection{Visual Comparisons}
\label{sec:Visual Comparisons}
We first visually evaluate the quality of the zero-shot text-to-image generation. The images are generated by text-to-image models solely using the label name as input.  \autoref{fig:visual_comparisons} provides a visual comparison between the generated images from our model and the real images across \(\epsilon\) settings on CIFAR10 task. We observe that the generated images can generally match their labels, while as the $\epsilon$ increases, the visual representation deviates. Therefore, how to set up the $\epsilon$ range in active learning would be a problem since larger $\epsilon$ may provide a more diverse generation but can deviate from its class or even lead to a crashed generation. In Section \ref{sec: abl exp}, we provide some quantitative evaluation on the setting of $\epsilon$.

\begin{figure}[!ht]
    \centering
    \includegraphics[width=\linewidth]{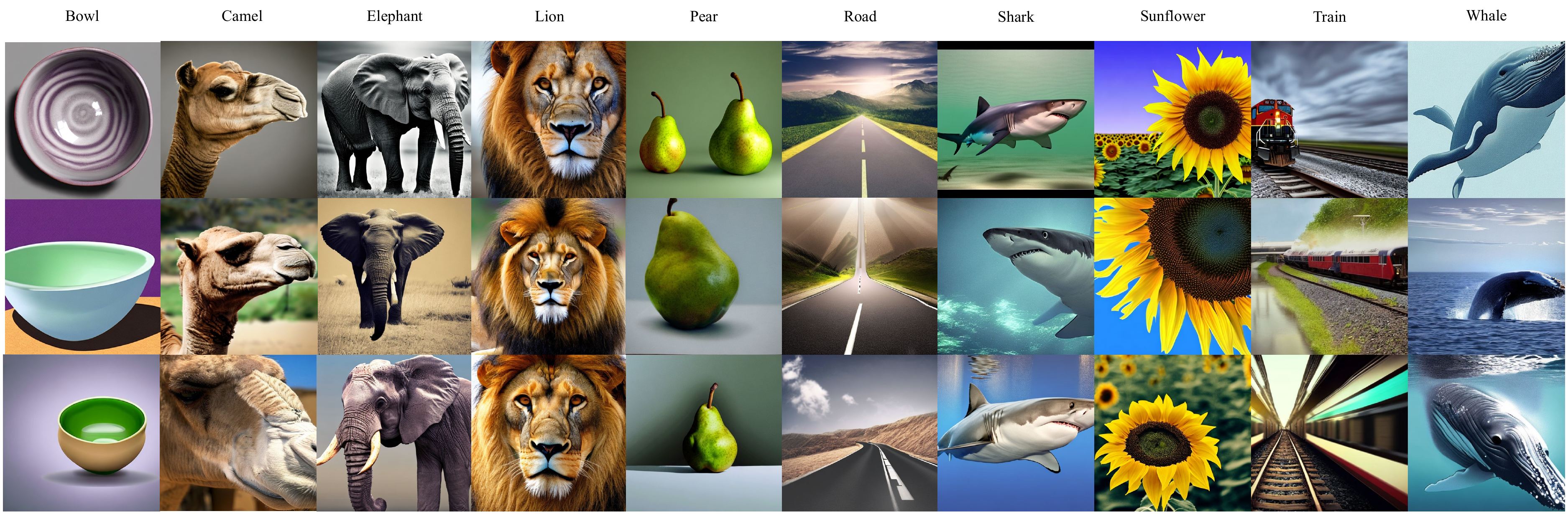}
    \caption{Generated Images for CIFAR100 Classification}
    \label{fig:example2}
\end{figure}

\begin{figure}[!ht]
    \centering
    \includegraphics[width=\linewidth]{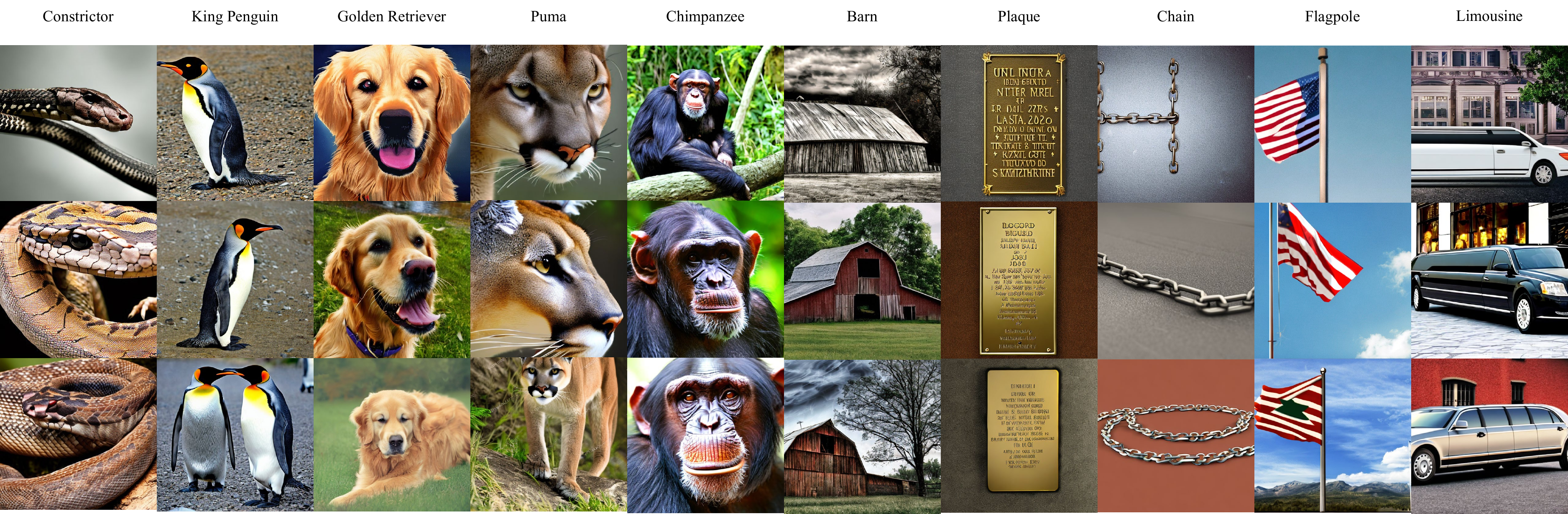}
    \caption{Generated Images for TinyImageNet Classification}
    \label{fig:example3}
\end{figure}

To further evaluate the quality of the T2I generation, we present some examples of the generated images for different datasets in the last cycle of active learning in Section \ref{sec: main exp}. We randomly select 3 images from 10 random classes of CIFAR10, CIFAR100, and TinyImageNet, presented in \autoref{fig:example1}, \autoref{fig:example2}, and \autoref{fig:example3}, respectively.